\documentclass[10pt,twocolumn,letterpaper,table]{article}
\usepackage[numbers,sort&compress,square,comma]{natbib}
\usepackage{cvpr}              
\usepackage{algorithm}
\usepackage{algorithmicx}
\usepackage[english]{babel}
\usepackage{amsthm,amsmath,amsfonts,amssymb}
\usepackage{epsfig}
\usepackage{amsfonts}
\usepackage{graphicx}
\usepackage{colortbl}
\usepackage{booktabs}
\usepackage{enumerate}
\usepackage{stfloats}
\usepackage{algpseudocode}

\usepackage{multirow}
\usepackage{graphicx}
\usepackage{color}
\newtheorem{theorem}{Theorem}

\newtheorem{definition}{Definition}

\usepackage{newfloat}
\usepackage{listings}
\usepackage{hyperref}
\hypersetup{
    colorlinks=true,
    linkcolor=blue,
    filecolor=blue,      
    urlcolor=blue,
    citecolor=green,
}

\def\paperID{7832} 
\def\confName{CVPR}
\def\confYear{2025}

\title{\textit{FP=\textbf{X}INT}: Representing Neural Networks via Low‑Bit Series Basis Functions}
\author{
    Boyang Zhang$^{1,2,3}$, Daning Cheng$^{1}$\thanks{Corresponding author}, Yunquan Zhang$^{1}$, Jiake Tian$^{2,5}$, Jing Li$^{4}$, Fangming Liu$^{2}$ \\ \\
    $^{1}$Institute of Computing Technology, Chinese Academy of Sciences, Beijing, China \\
    $^{2}$Pengcheng Laboratory, Shenzhen, China \\
    $^{3}$University of Chinese Academy of Sciences, Beijing, China\\
    $^{4}$Harbin Institute of Technology, Shenzhen, China \\
    $^{5}$South China University of Technology, Guangzhou, China\\ 
    zhangby01@pcl.ac.cn, \{chengdaning, zyq\}@ict.ac.cn, \\ mijiake@mail.scut.edu.cn, jingli.phd@hotmail.com, fangminghk@gmail.com
}

\begin{document}
\maketitle
\begin{abstract}
Post-Training Quantization (PTQ) converts pre-trained Full-Precision (FP) models into quantized versions without training. While existing methods reduce size and computational costs, they also significantly degrade performance and quantization efficiency at extremely low settings due to quantization noise. We introduce a deep model series expansion framework to address this issue, enabling rapid and accurate approximation of unquantized models without calibration sets or fine-tuning. This is the first use of series expansion for neural network quantization. Specifically, our method expands the FP model into multiple low-bit basis models. To ensure accurate quantization, we develop low-bit basis model expansions at different granularities (tensor, layer, model), and theoretically confirm their convergence to the dense model, thus restoring FP model accuracy. Additionally, we design AbelianAdd/Mul operations between isomorphic models in the low-bit expansion, forming an Abelian group to ensure operation parallelism and commutativity. The experiments show that our algorithm achieves state-of-the-art performance in low-bit settings; for example, 4-bit quantization of ResNet-50 surpasses the original accuracy, reaching 77.03\%. The code will be made public.
\end{abstract}    
\vspace{-0.6cm}
\section{Introduction}
The huge computational cost and memory usage in deep learning have been unable to meet the needs of resource-constrained devices. Researchers have studied model quantization techniques, which aim to convert high-precision parameters and activations into low-precision parameters and activations. Quantization methods are mainly divided into two categories, Quantization-Aware Training (QAT) and Post-Training Quantization (PTQ). QAT retrains the model on a labeled training dataset. Although the accuracy loss is small, it is time-consuming and computationally dense. The PTQ method is the mainstream of existing quantization methods, and only a small number of unlabeled samples are needed to quantize the trained model. PTQ does not require retraining and is suitable for rapid deployment on resource-constrained devices. The existing PTQ method can maintain good prediction when 8-bit quantization is used, but there is a gap with the full-precision model when 4-bit or 2-bit quantization is used.
This is because as the number of representation bits decreases, the error or noise caused by quantization increases, resulting in a decrease in network performance. The existing PTQ \cite{banner2018aciq,nagel2020up,li2021brecq,liu2023pd,shang2024enhancing} method uses various techniques to calibrate the quantized parameters to reduce the decline in accuracy, but it is difficult to reach the original accuracy under low-bit quantization and the cost time is long.
Moreover, the noise caused by quantization is inherent to the model, and the accuracy cannot be reduced by improving the hardware's computility.
\begin{figure}
\centering
\includegraphics[width=8.2cm,height=2.3cm]{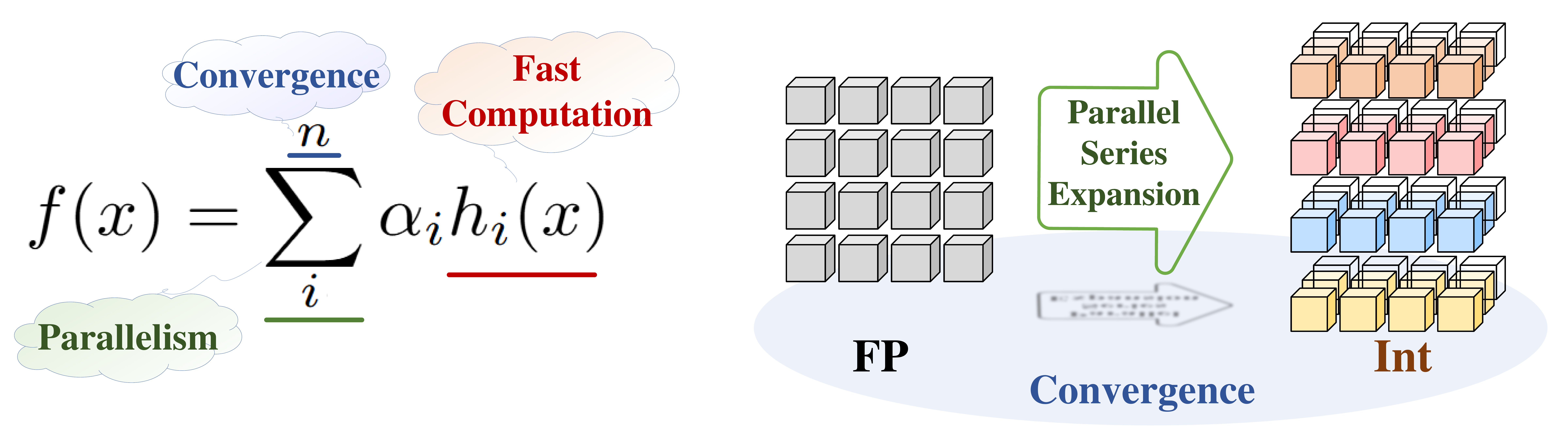}
\caption{The general form of series of a function $f(x)$. Usually, $h_i(x)$ is the computation-friendly function. The addition operation is a parallel-friendly operation. We expect the convergence speed of series expansion to be fast, which means $f(x)-\sum_{i=1}^n h_i(x)$ is small enough when $n$ is not too large.}
\vspace{-0.3cm}
\label{general form}
\end{figure}

In the context of PTQ, the goal is to preserve the original accuracy of the model even when using extremely low-bit quantization, which facilitates parallel computing. With advancements in hardware capabilities, the ability to maintain the original accuracy becomes even more pronounced. Series expansion offers a systematic mathematical approach to approximate functions with the desired level of accuracy. However, traditional series expansions, such as the Taylor and Fourier series, are not directly applicable to deep learning models, as they suffer from issues like the curse of dimensionality and high computational costs.

To address these challenges, inspired by function series expansion, we are the first to use series expansion in deep learning quantization methods.
We introduce a model series expansion framework that transforms the computationally inefficient original model into a computationally efficient basis function model class. As shown in Figure \ref{general form}, taking quantization as an example, we represent the original Full-Precision (FP) model through multiple low-bit basis models, achieving similar accuracy. Specifically, we progressively expand at both the tensor and layer levels.
First, we develop the tensor-level and layer-level low-bit expansions and theoretically demonstrate that the expansion converges exponentially to the original tensor or layer. Building on these tensor and layer expansions, we propose a global model series expansion strategy. For efficient operations within and across isomorphic base models, we design the operations AbelianAdd and AbelianMul, allowing the combination of these operations and the base models to form an Abelian group. The series expansion framework thus enables efficient interactions among isomorphic models and constructs a set of base models composed of low-bit or sparse calculations. Finally, we establish convergence to ensure both computational accuracy and efficiency.
Experiments show that our algorithm achieves state-of-the-art performance in low-bit settings. For instance, in ResNet-50, 4-bit quantization surpasses the original accuracy, reaching 77.03\%. Moreover, the parallelism in our expansion approach enables quantization speeds that exceed most existing methods.

The following are the main contributions of this paper: 1) We design the first model series expansion framework by series expanding the computationally inefficient original model into a computationally efficient basis function model class. 2) We design and construct tensor-level, single-layer-level, and model-level expansions and prove their convergence and efficiency. 3) Experiments show that our expansion framework has a significant improvement for extremely low-bit quantization.

\section{Target Formulation}
In this section, we introduce the paradigm of constructing series expansion in neural networks.
We aim to construct a function expansion class that can be quickly calculated to approximate the original function. The basic constructor of series expansion is as follows
\begin{equation}
    f(x) = \sum_{i=1}^{n} \alpha_i h_i(x),
    \label{series}
\end{equation}
where $f(x)$ is the original function, and $h_i(x)$ forms a class of basis functions that can be quickly calculated. This form uses linear combinations of basis functions to systematically approximate the original precision. At the same time, addition and multiplication construct an Abelian group, which is friendly to parallel computing such as AllReduce or MapReduce. In neural networks, $f(x)$, which is the model and we write as $model(sample)$ in the following parts, is expanded within the domain of the input data. We expect that for any sample in the dataset, the deep learning model, $model(sample)$, can be expanded into a specific series, i.e., $ model(sample) = \sum_{i=1}^{n} \alpha_i h_i(sample)$, to accelerate the inference speed of the model. However, for deep neural networks, traditional series expansions such as multidimensional Fourier series are not applicable. Traditional series face the curse of dimensionality, especially when the model input dimension is large, the number of terms in the traditional series will grow exponentially. Therefore, we hope to redesign the series expansion formula to adapt to existing neural networks
\begin{equation}
    model(sample) =  \sum_{\Xi_i}^n \hat{model}_i(sample).
    \label{target}
\end{equation}
Where $\sum_{\Xi_i}^n$ represents a new operation to replace the multiplication and addition operation in  Eq.~\ref{series}, which needs to form an Abelian group and is parallel-friendly. $\hat{model}_i$ represents a computationally efficient basis function. 
This process aims to construct a computationally efficient basis function that can ensure the convergence of the series on the abelian group. Thus the losslessness of the expansion formula is guaranteed.

\section{ Low-Bit Series Expansion Algorithm}
Since the computational performance of low-bit integers is higher than that of floating-point numbers. Therefore, we perform series expansion on the original neural network, and the expansion terms are constructed through low-bit quantized networks. We construct the low-bit expansion paradigm of the model level by level (tensor-layer-model) and prove that it converges to the unquantized network.

\subsection{Tensor Low-Bit Expansion}
We treat integer-quantized functions as the computational kernel for constructing basis functions. To illustrate the concept of quantized basis functions, we provide the following example. Suppose there are two data objects, $input_1$ and $input_2$, to be subjected to a computing operation, such as multiplication. After the quantization process \cite{yao2021hawq}, we have $Q_1 = \text{int}(\frac{input_1}{scale_1})$ and $Q_2 = \text{int}(\frac{input_2}{scale_2})$, and obtain
$Q_{output} = \text{int}(\frac{input_1 \times input_2}{scale_{output}}) \approx \text{int}( Q_1 \times Q_2 \times \frac{scale_1 \times scale_2}{scale_{output}})$, $scale_{output}$, $scale_1$, and $scale_2$ are precalculated scale factors that depend on the distributions of $input_1$, $input_2$, and the output. $Q_i$ is stored as a lower-precision data type, such as an integer. All $scale$ terms can be pre-calculated and established ahead of time. Then, during the entire inference process, only the $Q_i$ value needs to be calculated, which is fast. 

We set the tensor which will be quantized as $\mathbf{M}=(m_1,m_2,...,m_n),n\in\mathcal{N}$ and convert it into $m$ bits integer.
We analyze several methods of quantization one by one.
For symmetric quantization, the zero point of the quantized number is equal to the original vector, and for asymmetric, the quantization zero point is equal to the asymmetric point bias, represented by $bias*M_{nsy}$. For (non-) saturated quantization, $clip^{+/-}$ is used, and all $v_i>clip^+$ or $v_j<clip^-$ are set to $clip^{+/-}$, where the resulting difference is represented by $M_{sa}$.
For different quantization methods, we construct series expansions to prove that they converge to the original neural network function.
We can use the following Theorem \ref{matrix expand} to expand the original dense float tensor $M$.

\begin{figure}
\centering
\includegraphics[width=7.9cm,height=6cm]{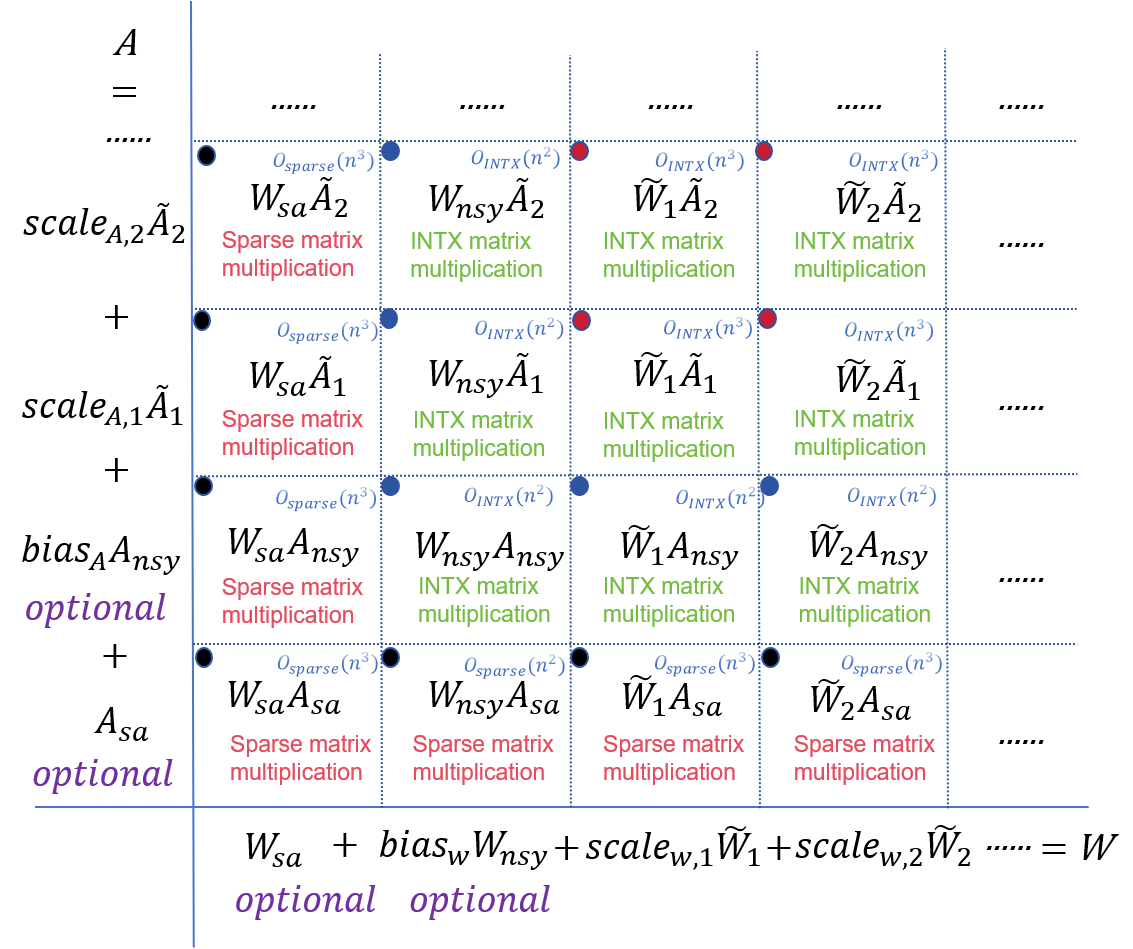}
\caption{The expansion of tensor multiplication, where $A$ and $W$ are $n*n$ tensor. The black point grid is produced by saturation quantization. The blue point grid is produced by non-symmetry quantization. The black and blue grid is optional. The red point is required by all quantization methods. The influence of black point grids is small in practice in the view of model performance. The blue grid computation complexity is small which is $O_{INTX}(n^2)$.  }
\vspace{-0.2cm}
\label{fig_rank:matrix multiplication expansion}
\end{figure}

\begin{theorem}
  $M=M_{sa}+bias*M_{nsy}+\sum_{i=1}^{n}scale_i*\widetilde{M}_{i}$,  where $M_{sa}$ is a sparse float tensor which is produced by saturation quantization,  $M_{nsy}$ is the tensor whose all elements are 1. $\widetilde{M}_{i}$ is the tensor whose all elements are INT(X) data type and $scale_i = 2^X*scale_{i+1}$.
\label{matrix expand}
\end{theorem}

\begin{proof}
\renewcommand{\qedsymbol}{}
We use the non-saturation and symmetry quantization methods as the base for analysis, and other quantization methods are based on this case.


In non-saturated symmetric quantization, for the tensor $M$, when applying the non-saturation and symmetry with a scaling factor $scale_1$ using $X$ bits, the result is $\widetilde{M}_1$. Then we define $R_1 = M-scale_1*\widetilde{M}_1$. $R_1$ has following properties that all elements in $R_0$ is smaller than $scale_1$. 
Next, the values $scale$ and $-scale$ as the max elements in non-saturation and symmetry quantization process in $R_1$'s quantization process and the $scale$ is calculated by $R_1$. We build $\widetilde{M}_2$ by non-saturation and symmetry quantizing $R_1$, and obtain $R_2 = M-scale_1*\widetilde{M}_1-scale_2*\widetilde{M}_2$. 
This procedure allows for parallel computation, and the relationship $scale_1 = 2^m * scale_2$ holds between the scales. Repeating above process, i.e., non-saturation and symmetry quantizing $R_i$, we have $M=\sum_{i=1}^n scale_i\widetilde{M}_i + R_{n+1}$. The max value in $R_{n+1}$ is smaller than $\frac{scale}{2^{Xn}}$. When $n\rightarrow\infty$, the maximum value in $R_{n+1}$ is converged to zero. The size of $\widetilde{M}_i$ is equal to the size of $M$.
For saturated symmetric quantization, the clipping operation in the saturated quantization method will produce errors, and we use another sparse tensor $M_{sa}$ to accommodate them. The sparse tensor $M_{sa}$ can be pre-calculated and is related to the data distribution, so $M_{sa}$ is a constant tensor. Then, this situation is converted to the unsaturated case.

When it comes to non-symmetry quantization methods, the main difference between non-symmetry and symmetry quantization methods is that in the non-symmetry quantization process, $v_i = bias + scale * q_i$ where $q_i$ is the quantization result for original tensor $v_i$. The $bias = \frac{v_{max}-v_{min}}{2} + v_{min}$ for non-saturation quantization and the $bias = \frac{clip^+ - clip^-}{2} + clip^-$ for saturation quantization. The $bias$ is the same for all elements for the tensor. Thus, $M = bias*M_{nsy}+ scale_1*\widetilde{M}_1 +R_1$, where $M_{nsy}$ is the tensor whose size is equal to the size of $M$. The method of processing $R_i$ is the same as the non-saturation and symmetry quantization case. Above proof also gives the process of building $M_{sa}$, $M_{nsy}$, $\widetilde{M}_i$.
\end{proof}

\subsection{Single-layer Low-Bit Expansion}
For a single-layer low-bit expansion in the neural network, we mainly focus on multiplication based on tensor low-bit expansion for the current deep learning model, the main computation kernel is tensor multiplication and it is the bottleneck of model computation performance.

Based on Theorem \ref{matrix expand}, replacing the $M$ tensor with weights $W$ and activations $A$, we can expand $W$ and $A$ can be expanded into $W=W_{sa}+bias_{w}*W_{nsy}+\sum_{i=1}^{n}scale_{w,i}*\widetilde{W}_{i}$ and  $A=A_{sa}+bias_A*A_{nsy}+\sum_{i=1}^{n}scale_{A,i}*\widetilde{A}_{i}$. We treat $bias$ as $scale_0$, $scale_{-1} = 1$, $W_{nsy}$ as $\widetilde{W}_0$, $W_{sa}$ as $\widetilde{W}_{-1}$, $A_{nsy}$ as $\widetilde{A}_0$, and $A_{sa}$ as $\widetilde{A}_{-1}$ for better description.

Then we have the following Eq.~\ref{mm low bit expansion}, and we term it as tensor multiplication low-bit expansion
\begin{equation}
\begin{aligned}
WA= \sum_{i,j\in[-1,n]} scale_{W,i}scale_{A,j}\widetilde{W}_i\widetilde{A}_j.
\label{mm low bit expansion}
\end{aligned}
\end{equation}

We also use Figure \ref{fig_rank:matrix multiplication expansion} to give a direct insight into tensor multiplication low-bit expansion.

For a single layer of the neural network $Layer(W, A)$, we construct the low-bit expansion of the layer as follows: 1) Expanding the original tensor multiplication. 2) Choosing one term of low-bit tensor multiplication, $scale_{W,i}scale_{A,j}W_iA_j$, as a kernel, build the layer $Layer(W_i,A_j)$. 
Then we have single-layer low-bit expansion as follows
\begin{equation}
\begin{aligned}
Layer(W,A) =  \sum_{i,j\in[-1,n]} scale_{w,i}scale_{A,j}Layer(W_i,A_j).
\label{mmlayerexpansion}
\end{aligned}
\end{equation}
We use $\widetilde{ layer}_{i,j}$ to indicate $Layer(W_i,A_j)$, and use $\hat{Layer}_{i,j}$ to indicate $ scale_{w,i}scale_{A,j}Layer(W_i,A_j)$. 

Similarly, due to the properties of tensor multiplication and the convergence of single-layer low-bit expansion, we can easily obtain the convergence of single-layer expansion. 
\subsection{Model Low-Bit Expansion}
The deep model consists of multiple single layers. We want to expand it into $t^2$ basis model. We can expand all the single layers whose calculation kernel is multiplication into low-order series items, and duplicate the remaining layers into basis models. Therefore, for the entire deep model, we can follow the steps below to build a low-order model: 1) Expand all tensor multiplications layer by layer. 2) For other layers, copy it into the basis model and the output of them and multiply $\frac{1}{t^2}$. 3) $\hat{model}_{i,j}, i,j\in[-1,n]$ is the the model, which always chose the $\hat{Layer}_{i,j}$ item of the low-order expansion of the 
original layer of the original model, and then perform multiple base model operations.

\begin{figure}
\centering
\includegraphics[width=8.4cm,height=3.9cm]{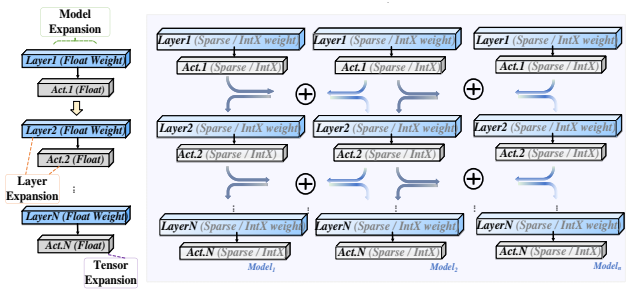}
\caption{Our series expansion at different levels and specific operation details. Finally, the FP model is expanded into the sum of multiple INT models.}
\vspace{-0.3cm}
\label{fig:model expansion}
\end{figure}

\textbf{Establishing the Abelian Group. } 
Unlike the series expansion of single layers and tensors, in the entire neural network, we prioritize designing the base model operations as well as parallelizing multiple series extension models. The original multiplication-addition cascade limits the parallelism and efficiency of the model series expansion. Therefore, we introduce basis functions and new operations to form an Abelian group, as the binary operations in an Abelian group align with the reduction operations used in AllReduce.

First, we expand the addition to adapt to neural networks. We define an operation AbelianAdd ({\Large$ \uplus$}) to construct an Abelian group.
\begin{definition}
{\Large$ \uplus$}: In neural networks, the output of each layer is multiplied by the scale, added, and broadcast to the input of the next layer.
\end{definition}
This operation applies to current neural networks because they are composed of layers with parameters and outputs, and AbelianAdd builds an Abelian group. We use $\sum_{\uplus,i=0}^{n}$ to describe the multi-AbelianAdd operations.
Similarly for multiplication, because the the parameter $w$ in layer $j$ is the integer multiple of $scale_j$, thus, the $w = scale_j * \widetilde{w}$ where $\widetilde{w}$ is the low-bit integer matrix.

The AbelianAdd satisfies the following property
\begin{align}
    &Model(\mathbf{W_1},\mathbf{A},sample)  \uplus   Model(\mathbf{W_2},\mathbf{A},sample) \notag \\
    = &Model(\mathbf{W_1}+\mathbf{W_2},\mathbf{A},sample), \label{pp1}
\end{align}
\begin{align}
    &Model(\mathbf{W},\mathbf{A_1},sample)  \uplus   Model(\mathbf{W},\mathbf{A_2},sample) \notag \\
    = &Model(\mathbf{W},\mathbf{A_1}+\mathbf{A_2},sample), \label{pp2}
\end{align}

\noindent where $\mathbf{W}=[W_1,W_2...]$ and $\mathbf{A}=[A_1,A_2,...]$.  $W_i$,$A_i$ is the weight and activation $i$th layer. $Model(\mathbf{W},\mathbf{A},sample)$ means the model whose weight is $\mathbf{W}$, activation is $\mathbf{A}$, and the input is $sample$. It is easy to see that $\hat{model}_{i,j} = Model(\mathbf{W}_i,\mathbf{A}_j,sample)$. 

Second, we define the AbelianMul $(\hat{*})$ to simplify description as following,
\begin{definition}
A vector $U=(u_1,u_2,...,u_k)$ AbelianMul the model means that the parameter $W_i$ in layer $i$  multiply the $u_i$,
$U \hat{*} model(w_i) = model(u_i*W_i)$.
\end{definition}
In the low-bit expansion of the model, the vector $U$ represents $scale$, and an Abelian group is formed between the operation (AbelianMul, AbelianAdd) and the isomorphic model. This allows the basis function model to be efficiently expanded. Note that the new operation does not change the model structure and can be regarded as a cascade expansion operation adapted to the neural network.


\textbf{Model Low-Bit Expansion. } 
In a multi-layer neural network, we construct a high computational efficiency $h_i$ in Eq.~\ref{series}  through the predefined AbelianAdd and AbelianMul and can avoid the dimensionality curse.
We use the isomorphic neural network but all parameters and input of the layer are $scale_j$'s integer multiple in layer $j$ or tensor matrix as $h_i$. 
Thus, we have the deep learning model's low-bit expansion that for any deep learning model $model$. We reconstruct Eq. \ref{target} ($ model(sample)=\sum_{\Xi_i}^n \hat{model}_i(sample)$),
\begin{theorem}  
  Form any locally continuous deep learning model $model$, whose core kernel is matrix multiplication or can be converted into multiplication,  weight is $\mathbf{W}$ and activation $\mathbf{A}$ we have the following expansion:  
  \begin{equation}  
  \begin{aligned}  
      model &= \sum_{\uplus,i,j \in[-1,n]} \hat{model}_{i,j} \\
      &=  \sum_{\uplus,i,j \in [-1,n]} \mathbf{scale_{i,j}} \hat{*}\widetilde{model}_{i,j},  
  \end{aligned}  
  \end{equation}  
  where $\hat{model}$ is the model whose parameters in a layer are sparse or integer multiple for $scale$ in $\mathbf{scale}$, and $\widetilde{model}$ is the model whose parameters in a layer is sparse or low-bit integer.  
\label{model expansion}  
\end{theorem}  

\begin{proof}  
\renewcommand{\qedsymbol}{}  
Based on the Eq. \ref{pp1} and Eq. \ref{pp2}, we have   
\begin{equation}  
\begin{aligned}  
    &\sum_{\uplus,i,j \in[-1,n]} \hat{model}_{i,j} \\
    &=\sum_{\uplus,i,j \in[-1,n]} Model(\mathbf{W}_i,\mathbf{A}_j,sample)\\
    &=Model\left(\sum_{i \in[-1,n]}\mathbf{W}_i,\sum_{j \in[-1,n]} \mathbf{A}_j,sample\right).  
\end{aligned}  
\end{equation}  
\end{proof}  
Based on Theorem \ref{matrix expand}, we know that $\sum_{i \in[-1,n]}\mathbf{W}_i$ is convergent to $\mathbf{W}$ exponentially and $\sum_{j \in[-1,n]}\mathbf{A}_j$ is convergent to $\mathbf{A}$ exponentially. So, based on the locally continuous property, \(model = \sum_{\psi, i, j \in [-1, n]} model_{i, j}\).

Based on Theorem \ref{model expansion}, the inference process can be expanded into the pattern shown in Figure \ref{fig:model expansion}. In this pattern, all $\widetilde{model}$ are low computation resource models whose main computation kernel is low-bit integer or sparse. Each layer computes the quantized activation and tensor multiplication independently, which is shown in the following section, and all reduce the output of each  $\widetilde{model}$.  

\section{The Complexity Analysis of Low-Bit Series Expansion}
Theorem \ref{model expansion} shows the equivalence of the model, but from the perspective of model performance and parallel equivalence, the computational complexity can be further reduced.

\textbf{The Weight Expansion Upper Bound. }
From the low-order expansion of the tensor multiplication of $WA$, when $W$ and $A$ are expanded by $t$ term, we have to calculate $t^2$ matrix multiplication to obtain accurate results in Figure \ref{fig_rank:matrix multiplication expansion}.

However, in post-training quantization, we do not have to compute $t^2$ low-bit matrix multiplications.  For a well-trained model on a dataset and its loss function $\ell$, the loss function's gradient of weight $W$ is zero, i.e., $\frac{\partial \ell}{\partial W}$. Thus, when $W$ is introduced the error, the loss function value $\ell(model(W)) - \ell(model(W+error)) = \frac{\partial \ell}{\partial W}*error = 0$. So from the perspective of the loss function, it makes no sense to expand the $W$ matrix by too many terms. For the matrix expansion $W=W_{sa}+bias_{w}*W_{nsy}+\sum_{i=1}^{n}scale_{w,i}*\widetilde{W}_{i}$ by INTX, the maximum $n$ should satisfy the influence of $scale_n * 2^X$ should be described by total differential. The above condition means that the $scale_n * 2^X$ should be small enough, 
empirically,  $scale_n * 2^X < 10^{-2}$, and usually, the weight parameter only has to expand 2 or 3 terms. Furtherly,  in practice the influence of $A_{sa}$ and  $W_{sa}$ on the loss function is small.

Based on the above analysis, we only need to expand the $A$ matrix multiple times to obtain better performance, and we have to compute $O(t)$ low-order matrix multiplications instead of  $O(t^2)$ low-order matrix multiplications, i.e.,  $model =  \sum_{\uplus, i \in [0,k],j \in [0,t]} \mathbf{scale_{i,j}} \hat{*}\widetilde{model}_{i,j}$ and $k$ is small instead of  $model =  \sum_{\uplus, i,j \in [0,t]} \mathbf{scale_{i,j}} \hat{*}\widetilde{model}_{i,j}$ from the view of loss function.

\textbf{The Computation Complexity of $M_{nsy}$ Multiplication. }We can see the $M_{nsy}$ is the tensor whose all elements are one, which means $M_{nsyn}$ is are low-rank matrix. In fact, $M_{nsyn} = \mathbf{one}^T\mathbf{one}$, where $\mathbf{one}$ is a vector and $\mathbf{one} = (1,1,...,1)$. So, for many matrix multiplication $MM_{nsyn} = M\mathbf{one}^T \mathbf{one} = (M\mathbf{one}^T)\mathbf{one}$. The computation complexity of the latter process is $O(n^2)$.

\begin{table*}[ht]
\renewcommand{\arraystretch}{1.15}
\setlength{\tabcolsep}{11.5pt}
\vspace{-0.1cm}
\caption{
Comparison of our algorithm with various post-training quantization algorithms. Results are averaged over multiple rounds. Our results have significant advantages in extremely low-bit quantization.}\scalebox{0.83}{
\begin{tabular}{cccccccc}

\cline{1-8}
\textbf{Methods} & \textbf{Bits(W/A)}    & \textbf{ResNet-18}                    & \textbf{ResNet-34} & \textbf{ResNet-50}                    & \textbf{ResNet-101} & \textbf{RegNetX-600MF} & \textbf{Inception-V3} \\ \cline{1-8}
Full Prec.       & 32/32                 & 71.01                                 & 73.3               & 76.63                                 & 77.3               & 73.52                  & 77.4               \\ \cline{1-8}
ACIQ~\cite{banner2018aciq}         &                       & 67                                    & 69.1               & 73.8                                  & -                  & -                      & 60.4                 \\
DFQ~\cite{nagel2019data}              &                       & 57.1                                  & -                  & 64.5                                  & -                  & 57.71                  & -                    \\
AdaRound~\cite{nagel2020up}         &                       & 67.96                                 &                    & 73.88                                 & -                  & -                      & -                    \\
AdaQuant~\cite{hubara2020improving}        &                       & 67.4                                  & 70.3               & 73.7                                  & 74.4               & 68.2                   & 72.6                 \\
Seq-AdaQuant     &                       & 69.4                                  & 71.7               & 75.1                                  & 75.5               & -                      & 73.4                \\
PD-Quant~\cite{liu2023pd}       &                       & 69.3                                  & -                  & 75.09                                 & -                  & 70.95                  & -                     \\
\rowcolor[rgb]{0.8,0.8,0.8}\textbf{Ours}    & \multirow{-7}{*}{4/4} & {\textbf{70.37}} & \textbf{72.75}         & { \textbf{77.03}} & \textbf{76.6 }              & \textbf{71.8}          & \textbf{76.09}        \\ \cline{1-8}
LAPQ~\cite{nahshan2021loss}           &                      &0.18                                    & 0.14               & 0.17                                 & -              & 0.12                      & -               \\ AdaRound~\cite{nagel2020up}         &                       &0.11                                &   -                 & 0.12                                & -                  & -                      & -                    \\
CL-Calib~\cite{shang2024enhancing}        &                       & 65.14                                 & -                  & 70.92                                 & -                  & 64.5                 & -                     \\ 
\rowcolor[rgb]{0.8,0.8,0.8}\textbf{Ours}    & \multirow{-4}{*}{2/4} & {  \textbf{70.26}} & \textbf{71.95}         & {  \textbf{74.23}} & \textbf{75.10 }              & \textbf{70.04}          & \textbf{74.27}
 \\

\cline{1-8}
ACIQ~\cite{banner2018aciq}           &                       & 0.12                                  & 0.2                & 0.11                                  & 0.21               & -                      & 0.11                \\ 
BRECQ~\cite{li2021brecq}           &                       & 42.54                                 & -                  & 29.01                                 & -                  & 3.62                   & -                   \\ 
AdaQuant~\cite{hubara2020improving}        &                       & 0.11                                  & -                  & 0.12                                  & 0.14               & -                      & 0.13                   \\ 
PD-Quant~\cite{liu2023pd}        &                       & 53.08                                 & -                  & 56.98                                 & -                  & 55.13                  & -                   \\ 
CL-Calib~\cite{shang2024enhancing}        &                       & 54.45                                 & -                  & 58.3                                  & -                  & 56.39                  & -                     \\ 
\rowcolor[rgb]{0.8,0.8,0.8} \textbf{Ours}    & \multirow{-6}{*}{2/2} & {  \textbf{59.14}} & \textbf{63.58}         & {  \textbf{62.13}} & \textbf{61.64}    & \textbf{59.60}          & \textbf{49.87}        \\ \cline{1-8}
\end{tabular}} \label{taball}
\vspace{-0.3cm}
\end{table*}
\textbf{The Parallelization of Computing $\widetilde{M}_i$. \label{cha:para}} As we can see from the proof of Theorem \ref{matrix expand}, we only use the non-saturation quantization process at the first time to produce the $\widetilde{M}_1$. In the computing process of $\widetilde{M}_i,i>1$, we set the maximum element of $R_i,i>1$ as $scale_{i-1}$. We use this setting because we the property, $scale_i = 2^X scale_{i+1}$, to parallelize the computing process. 

Begin with the non-saturation symmetry case, we can easily gain that the $(i,j)$ element in $\widetilde{M}_k$ can be calculated as $\widetilde{M}_k(i,j)=INTX(\frac{M(i,j)}{scale_{k}}) - INTX(\frac{M(i,j)}{scale_{k-1}})*2^X$. The cases of saturation and non-symmetry cases can be converted into non-saturation symmetry cases by adding $M_{nsy}$ and $M_{sa}$. 



\section{Experiments}
\subsection{Experiment Details} We conduct series expansion quantization experiments on various models on Imagenet~\cite{deng2009imagenet} and NLP tasks~\cite{rajpurkar2016squad, williams2017broad}, including ResNet~\cite{he2016deep}, RegNet~\cite{radosavovic2020designing}, etc. Our method sets hyperparameters consistently on all models and quantizes channel by channel. The basis function selects the class of integer quantization functions. 
We determine the value of $clip^{+/-}$ during saturation quantization to minimize the impact of $M_{sa}$.
In non-Saturated quantization, we use the expected quantization noise in the Laplace distribution as the clipping function. We quantify weights and activations separately. Then the activations of all base models are broadcast and quantized. After low-bit reasoning, All-Reduce outputs all low-bits. For all PTQ experiments, we set the first and last layer quantization to 8-bits. Our code is based on pytorch. All experiments are conducted on a single NVIDIA A800 GPU.

\subsection{Comparison to State-of-the-arts }
\textbf{ImageNet.} We comprehensively compare our algorithm with other PTQ algorithms under multiple-bit settings. In each bit-setting basis function, our algorithm achieve performance improvement. As shown in Table \ref{taball}, when the model is quantized to W4A4, our algorithm is almost consistent with full precision. Even in ResNet-50, our algorithm exceeds full precision. This shows that our algorithm can perform series expansion to replace floating point functions with int-basis functions.

For the more challenging W2A2 quantization, our algorithm achieves the best performance in all networks. Existing quantization methods are almost unusable at W2A2. For example, in the more complex ResNet-101 and Inception-V3, methods such as ACIQ are difficult to perform extremely low-bit quantization. Our method still works even after extremely low-bit quantization, with an accuracy of more than 60\%.
This demonstrates the efficiency of series expansion. For weights and activations quantized at W2A4, our method still delivers the best performance. When the weights are quantized to 2-bits, the accuracy degradation is minimal, indicating that the weights can tolerate larger quantization errors. This further validates the effectiveness of our approach for weight expansion, showing that repeated weight expansion is not necessary.

\begin{table}[ht]
\renewcommand{\arraystretch}{1.15}
\caption{
Performance of the algorithm under different bit-settings, our \textit{ INT} quantization accuracy is close to \textit{FP}.}
\scalebox{0.85}{
\begin{tabular}{lllllllllllll}
\cline{1-6}
\multicolumn{1}{c|}{\textbf{Model}} & \multicolumn{5}{c}{\textbf{ResNet-18}}                      \\ \cline{1-6}
\multicolumn{1}{c|}{Bits}  & \multicolumn{1}{c}{W3A3} & \multicolumn{1}{c}{W2A4}  & \multicolumn{1}{c}{W4A2}  & \multicolumn{1}{c}{W8A8}  & \multicolumn{1}{c}{W32A32} &    \\ \cline{1-6}
\multicolumn{1}{c|}{AdaQuant~\cite{hubara2020improving}}  & \multicolumn{1}{c}{60.09} & \multicolumn{1}{c}{0.11} & \multicolumn{1}{c}{-} & \multicolumn{1}{c}{-} & \multicolumn{1}{c}{71.01} &
\\ \cline{1-6}
\multicolumn{1}{c|}{QDrop~\cite{wei2022qdrop}}  & \multicolumn{1}{c}{65.56} & \multicolumn{1}{c}{64.66} & \multicolumn{1}{c}{57.56} & \multicolumn{1}{c}{-} & \multicolumn{1}{c}{71.06} &
\\ \cline{1-6}
\multicolumn{1}{c|}{BRECQ~\cite{li2021brecq} }  & \multicolumn{1}{c}{65.87} & \multicolumn{1}{c}{64.80} & \multicolumn{1}{c}{-} & \multicolumn{1}{c}{-} & \multicolumn{1}{c}{71.08} &\\ \cline{1-6}
\rowcolor[rgb]{0.8,0.8,0.8} \multicolumn{1}{c|}{\textbf{Ours}}  & \multicolumn{1}{c}{\textbf{68.8}} & \multicolumn{1}{c}{\textbf{70.26}} & \multicolumn{1}{c}{\textbf{60.57}} & \multicolumn{1}{c}{\textbf{71.00}} & \multicolumn{1}{c}{71.01}

\\ \cline{1-6}
\rowcolor[rgb]{0.8,0.8,0.8} \multicolumn{1}{c|}{Quant-Time}  & \multicolumn{1}{c}{5.334s} & \multicolumn{1}{c}{3.236s} & \multicolumn{1}{c}{4.456s} & \multicolumn{1}{c}{0.063s} & \multicolumn{1}{c}{-} \\ \cline{1-6}
\end{tabular}} \label{tabmix}
\vspace{-0.4cm}
\end{table}

\begin{table*}[ht]  
\centering  
\renewcommand{\arraystretch}{1.2}  
\setlength{\tabcolsep}{13pt}  
\caption{  
Comparison of our algorithm with different quantization methods and accuracy time comparison using mixed precision quantization. Our algorithm does not require fine-tuning (FT) and calibration sets.}  
\vspace{-0.3cm}  
\scalebox{0.80}{  
\begin{tabular}{cccccccc}  
\toprule  
\textbf{Models} & \textbf{Method} & \textbf{Bits (W/A)} & \textbf{Accuracy} & \textbf{Model Size} & \textbf{Training Data} & \textbf{Runtime} & \textbf{Calibration/FT} \\   
\midrule  
\multirow{5}{*}{\begin{tabular}[c]{@{}c@{}}\textbf{ResNet-18}\\ (FP:71.01)\end{tabular}}   
& ZeroQ~\cite{cai2020zeroq} & 4/4 & 21.71 & 5.81M & \textbf{0} & 0.008h & 1024 \\   
& DSQ~\cite{gong2019differentiable} & 4/4 & 69.56 & 5.81M & 1.2M & 100h & w/ FT \\   
& PACT~\cite{choi2018pact} & 4/4 & 69.2 & 5.81M & 1.2M & 100h & w/ FT \\   
& QDrop~\cite{wei2022qdrop} & 4/4 & 69.17 & 5.81M & \textbf{0} & 0.43h & 1024 \\   
& PD-Quant~\cite{liu2023pd} & 4/4 & 69.3 & 5.81M & \textbf{0} & 1.11h & 1024 \\   
\rowcolor[rgb]{0.8,0.8,0.8}   
& \textbf{Ours} & 4/4 & \textbf{70.37} & 5.81M & \textbf{0} & \textbf{0.39h} & \textbf{0, w/o FT} \\   
\rowcolor[rgb]{0.8,0.8,0.8}   
& \textbf{Ours} & 2/Mix(2/4/8) & 69.01 & \textbf{3.01M} & \textbf{0} & 2.5h & \textbf{0, w/o FT} \\   
\midrule  
\multirow{4}{*}{\begin{tabular}[c]{@{}c@{}}\textbf{MobileNetV2}\\ (FP:72.49)\end{tabular}}   
& DSQ~\cite{gong2019differentiable} & 4/4 & 64.8 & 2.26M & 1.2M & 192h & w/ FT \\   
& PACT~\cite{choi2018pact} & 4/4 & 61.4 & 2.26M & 1.2M & 192h & w/ FT \\   
\rowcolor[rgb]{0.8,0.8,0.8}   
& \textbf{Ours} & 4/4 & \textbf{71.1} & 2.26M & \textbf{0} & \textbf{1.94h} & \textbf{0, w/o FT} \\   
\rowcolor[rgb]{0.8,0.8,0.8}   
& \textbf{Ours} & 2/Mix(2/4/8) & 65.21 & \textbf{1.18M} & \textbf{0} & 5.5h & \textbf{0, w/o FT} \\   
\bottomrule  
\end{tabular}}  
\label{tabqat}  
\vspace{-0.3cm}  
\end{table*}

In Table \ref{tabqat}, we also show the comparison between the algorithm and different PTQ~\cite{cai2020zeroq} and QAT~\cite{gong2019differentiable, choi2018pact} methods. With the same model size, our quantized model can achieve higher accuracy and does not require additional data for fine-tuning. When using mixed precision quantization, our model can achieve a smaller size and maintain high accuracy. At the same time, our algorithm is parallelizable. Once the algorithm calculates the quantization parameters of each expansion remainder in parallel, no additional calculation is required during inference. In addition, our running time is better than the existing QAT and PTQ methods. Although our method expands multiple  INT models, our algorithm greatly reduces the running time through parallelism, which is less than the existing methods.
The average quantization time per epoch of our algorithm under different bit widths is shown in Table \ref{tabmix}. We also show the accuracy comparison when weights and activations use different bits. Our algorithm can flexibly use different bit widths to quantize the model, and the  INT8 accuracy is almost consistent with the full precision. For hardware that supports  INT8,  INT4, and even  INT2, our algorithm can speed up inference by nearly 40-60\%.

\textbf{SQuAD and MNLI.} As shown in Table \ref{tabblp}, in addition to testing on image data, we also validated the series expansion algorithm in natural language processing (NLP) tasks, applying it to the typical NLP model BERT. Compared to existing methods, our algorithm does not require additional data, whereas the QDROP method necessitates extra data examples. The series expansion algorithm outperforms existing methods and achieves significant performance improvements. This highlights the versatility and strong generalizability of our approach.
\begin{table}[!ht]
\renewcommand{\arraystretch}{1}
\setlength{\tabcolsep}{14pt}
    \centering
    \caption{Performance on NLP tasks compared with existing methods on W4A4.}
    \vspace{-0.1cm}
    \scalebox{0.83}{
    \begin{tabular}{c|c|c}
    \hline
        \textbf{Method} & \textbf{SQuAD1.1(F1)} & \textbf{MNLI(Acc mm)} \\ \hline
        Full Prec. & 88.42 & 84.57 \\ \hline
        AdaQuant~\cite{hubara2020improving} & 5.17 & - \\ \hline
        BRECQ~\cite{li2021brecq}  & 68.58 & 31.91 \\ \hline
        QDrop(no) & 75.97 & 69.19 \\ \hline
        QDrop~\cite{wei2022qdrop} & 77.26 & 71.43 \\ \hline
       \rowcolor[rgb]{0.8,0.8,0.8} \textbf{Ours} & \textbf{79.30} & \textbf{72.31} \\ \hline
    \end{tabular}} \label{tabblp}
    \vspace{-0.3cm}
\end{table}

\subsection{Ablation}
\textbf{Ablation of Non-saturation.}
As shown in Figure \ref{fig_rank}a, we perform saturation and non-saturation ablations on the quantization basis functions. When the clipping function is not used, the quantization performance on the four models is slightly worse, but still better than existing methods. When using the Laplace clipping function, the accuracy of our algorithm is close to full accuracy.
 \begin{figure}
\centering
\includegraphics[width=8.2cm,height=3.3cm]{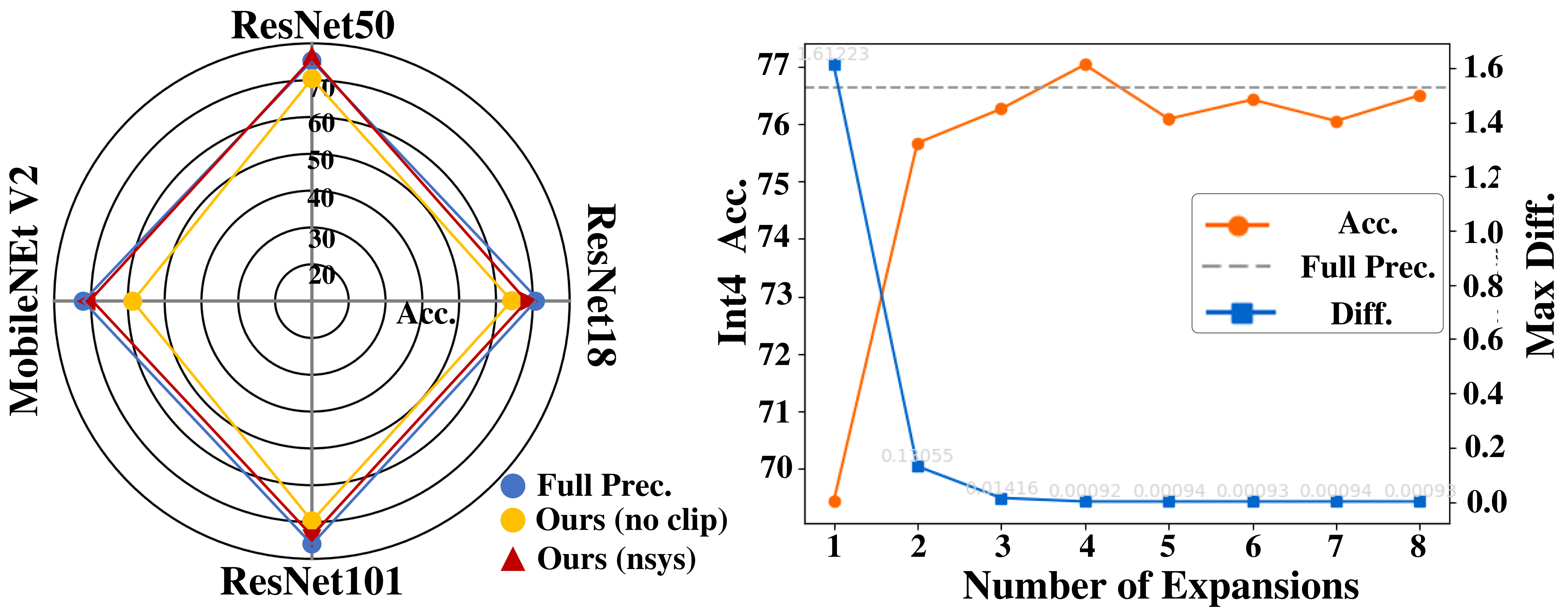}
\vspace{-0.1cm}
\caption{The left sub-figure(a) shows our experiments on saturation and asymmetric quantization. The right sub-figure(b) shows the changes in loss and accuracy as the number of expansions increases.}
\vspace{-0.5cm}
\label{fig_rank}
\end{figure}

\textbf{Ablation of the Expansion.}
As shown by the orange line in Figure \ref{fig_rank}b, when the number of expansions increases, the accuracy of ResNet-50 increases and gradually approaches the accuracy of the original model.
This is due to the decrease in the difference between the activations of the original model and the quantized model, as shown by the blue line.
When the number of expansions is 4, the model has the best accuracy.
When the activations continue to expand and the number exceeds 5, the maximum difference will continue to decrease, but the accuracy will change slightly, while the calculation time will increase significantly. Therefore, in the specific implementation, we stipulate that when the maximum difference is less than $10^{-4}$, the number of expansions is optimal. It is worth mentioning that the error caused by quantization on weights is smaller than that on activations, so the weights are not expanded more than twice.
This also verifies theoretical analysis.

As shown in Table \ref{tabab}, we ablate the accuracy of only expanding weights or activations. When the weights and activations are expanded the same number of times, it can be seen that when only the activations are expanded, the impact on the accuracy is greater than that of only the weights. Therefore, series expansions of the activation quantization are necessary to help improve the accuracy.

\begin{table}[ht]
\centering
\renewcommand{\arraystretch}{1.1}
\setlength{\tabcolsep}{22pt}
\vspace{-0.2cm}
\caption{
Accuracy ablation of only expanding weights and only expanding activations for ResNet model under  INT4 quantization.}
\vspace{-0.2cm}
\scalebox{0.8}{
\begin{tabular}{lllllllllllll}
\cline{1-4}
\multicolumn{1}{c|}{\textbf{Model}}    & \multicolumn{1}{c|}{\textbf{onlyA}} & \multicolumn{1}{c|}{\textbf{onlyW}} & \multicolumn{1}{c}{\textbf{Ours}}  \\ \cline{1-4}
\multicolumn{1}{c|}{\textbf{ResNet-18}} & \multicolumn{1}{c|}{69.06}          & \multicolumn{1}{c|}{68.12}          & \multicolumn{1}{c}{\textbf{70.37}}       \\ \cline{1-4}
\multicolumn{1}{c|}{\textbf{ResNet-50}} & \multicolumn{1}{c|}{75.12}          & \multicolumn{1}{c|}{72.87}          & \multicolumn{1}{c}{\textbf{77.03}}       \\ \cline{1-4}
\end{tabular}} \label{tabab}
\vspace{-0.5cm}
\end{table}

\subsection{Discussion}

\textbf{Series Expansion $\neq$ Ensemble.}
Our approach fundamentally differs from the combination of multiple integer (INT) models. The ensemble of multiple INT models cannot converge to the original model, while our method can converge to the unquantized model, focusing on computing the series expansion of different low-bit models. Ensemble models combine the parameters of multiple similar quantized models but do not achieve convergence. Experimental results show that when existing ensemble methods are used to combine multiple INT models, there is no performance gain, and the performance actually decreases.

\textbf{Generalization to Large Language Models.}
As shown in Table \ref{tabllm}, although our work primarily focuses on small models for specific tasks, we still explore its application in LLMs. Following the quantization settings of W4A16, our approach remains effective in LLMs. On MMLU, our method achieves accuracy nearly equivalent to the original model and surpasses it in certain areas, demonstrating the efficiency and superiority of series expansion. In the future, we will focus on the quantization of LLMs.
\begin{table}
\centering
\renewcommand{\arraystretch}{1.1}
\setlength{\tabcolsep}{8pt}
\caption{
Performance comparison of different 4-bit quantization methods for large language models of different sizes.}\vspace{-0.2cm}
\scalebox{0.8}{
\begin{tabular}{c|cccccccccccc} 
\cline{1-6}
\multirow{2}{*}{Method} & \multicolumn{5}{c}{MMLU ~\cite{hendrycks2020measuring}}                                                                                         &                      &                      &                      &                      &                      &                      &                       \\ 
\cline{2-6}
                        & \textbf{Hums. }               & \textbf{STEM}                 & \textbf{Social}               & \textbf{Other}               & \textbf{Avg.}                                 \\ 
\cline{1-6}
LLaMA3-8B               & \textbf{59.0}          & 55.3                 & 76.0                   & \textbf{71.5}        & \textbf{64.8}                   \\
Normal                  & 56.8                 & 52.9                 & 73.6        & 69.4 &62.5                             \\
\rowcolor[rgb]{0.8,0.8,0.8} \textbf{Ours}                    & 58.9                 & \textbf{55.6}        & \textbf{76.1}        & 71.3                 & 64.7                                 \\ 
\cline{1-6}
LLaMA2-7B \cite{touvron2023llama}              & 43.1                 & 36.4                 & 51.6                 & 52.3                 & 45.7                                     \\
GPTQ~\cite{frantar2022gptq}                    & -                    & -                    & -                    & -                    & 42.8                               \\
QAT~\cite{liu2023llm}                    & -                    & -                    & -                    & -                    & 42.7                                   \\
\textbf{Ours}                    & 42.8                 & \textbf{36.8}        & \textbf{51.9}        & \textbf{52.3}        & \textbf{45.7}                          \\ 
\cline{1-6}
Qwen2.5-3B \cite{qwen2}             & 56.6                 & \textbf{61.5}        & 76.8                 & 70.3                 & \textbf{65.3}                            \\
\rowcolor[rgb]{0.8,0.8,0.8} Ours                    & \textbf{56.6}        & 61.4                 & \textbf{76.8}        & \textbf{70.3}        & 65.2                                   \\ 
\cline{1-6}
\multicolumn{1}{l}{}    & \multicolumn{1}{l}{} & \multicolumn{1}{l}{} & \multicolumn{1}{l}{} & \multicolumn{1}{l}{} & \multicolumn{1}{l}{} & \multicolumn{1}{l}{} & \multicolumn{1}{l}{} & \multicolumn{1}{l}{} & \multicolumn{1}{l}{} & \multicolumn{1}{l}{} & \multicolumn{1}{l}{} & \multicolumn{1}{l}{} 
\end{tabular}}\label{tabllm} \vspace{-0.8cm}
\end{table}

\textbf{Scalability and Basis Function Diversity.}
Our goal is to leverage the higher throughput of INT processing units to achieve higher inference speed while using the same hardware to serve the model. Compared to the original floating-point (FP) model with the same number of hardware units, our method quantizes the model with arbitrary bits (e.g., 2-bit) while maintaining high accuracy. The quantization results under INT8 further demonstrate that our algorithm can achieve near full precision with only integer data types and achieve about 4 times the throughput of FP32 on INT8 processing units. Therefore, our method enables inference acceleration on specific hardware that supports accelerated operators, showing great potential for hardware design and implementation of pure INT-based operators.

Additionally, our method offers a conventional approach to balancing computational efficiency and model performance. Unlike existing complex quantization methods, our solution focuses on applying series expansion to neural networks and incorporates traditional quantization methods.
It is worth mentioning that existing quantization schemes can be embedded in our series expansion algorithm to further improve performance.
Finally, the selection of basis functions can be diversified, meeting both computational efficiency and hardware friendliness requirements. In this study, we used integer models as basis functions, but in practice, we could also use sparse models or other computation-friendly models as the base functions.

\section{Related Work}
\textbf{Series Expansion and Parallel Processing.}
Series expansion~\cite{censor1983finite} is a method of representing a function as an infinite series, which can be used to approximate complex functions. The expanded functions are called basis functions, which can express or approximate other functions through linear combinations.
Traditional basis functions are trigonometric functions and polynomial functions, such as Fourier expansion and Taylor expansion. Because today's computers are good at calculating polynomial functions and matrix operations.
So we can get computational benefits from series expansion by converging functions. However deep models are designed with a serial architecture, which slows down computational efficiency.
So the operation of the expansion terms needs to be processed in parallel. AllReduce is a key operation in parallel computing.
To parallelize a process through AllReduce, its operation must be an Abelian group.
This means that the reduction operation needs to satisfy both the commutative law and the associative law, so the order of the combined values does not affect the result. The reduction operation in AllReduce corresponds to the binary operation in the Abelian group. Series expansion has the advantages of parallelism and accurate approximation. Our algorithm is an extension of traditional expansion in deep learning, overcoming problems such as the curse of dimensionality and inheriting the advantages of parallelism and accuracy.

\textbf{Quantization} is one of the most popular techniques for compressing neural networks, generally divided into two main approaches: Quantization-Aware Training (QAT) and Post-Training Quantization (PTQ). QAT incorporated quantization during the network training phase, whereas PTQ applied quantization after training. PTQ was widely used in network deployment due to its low time and computational resource requirements.
HAWQ-v2~\cite{dong2020hawq} used the trace of a layer's Hessian matrix as an indicator of the layer’s sensitivity. The work~\cite{2019Optimizing} formulated the mixed-precision problem for weights and activations as a Lagrangian optimization problem, where the solution determined optimal precision allocation across weights and activations.
AdaQuant~\cite{hubara2020improving} optimized its parameters on a calibration set to minimize quantization error for each layer individually. PD-Quant ~\cite{liu2023pd} alleviated the overfitting problem in PTQ caused by the small number of calibration sets by adjusting the distribution of activations. Shang et al.~\cite{shang2024enhancing} introduced mutual information into PTQ calibration to optimize the quantization parameters. These methods focus on the quantization parameters to alleviate the error, but cannot achieve good results at very low bits and run slowly. Our method achieves performance and speed improvements at low bits and does not require calibration sets and fine-tuning.
\vspace{-0.1cm}

\section{Conclusion}
This paper proposes a deep model series expansion framework that aims to replace the computationally inefficient original model with multiple computationally efficient basis function models. We use it in PTQ to expand the FP model to the INT version and prove its convergence. Theoretical and experimental results show that the algorithm improves the parallel capability of the model and guarantees the performance of the low-bit model without the need for calibration sets and fine-tuning.

    \small
    \bibliographystyle{ieeenat_fullname}
    \bibliography{main}


\end{document}